\documentclass[letterpaper, 10 pt, conference]{ieeeconf}
\IEEEoverridecommandlockouts
\usepackage{cite}
\usepackage{xcolor}
\usepackage{mathtools}
\usepackage{footnote}
\usepackage{cancel}
\usepackage{xfrac}
\usepackage{amsmath,amssymb,amsfonts}
\usepackage{algorithm}
\usepackage{algorithmicx}
\usepackage{algpseudocode}
\usepackage[hidelinks]{hyperref}
\usepackage{siunitx}
\usepackage{graphicx}
\usepackage{float}
\usepackage{tikz}
\usepackage{units}
\usepackage{textcomp}
\usepackage[utf8]{inputenc}
\usepackage[english]{babel}

\usetikzlibrary{arrows.meta,automata,calc,intersections,positioning}
\newtheorem{theorem}{Theorem}
\newtheorem{lemma}{Lemma}

\newcommand{\myparagraph}[1]{\textbf{#1.}}
\renewcommand{\vec}[1]{\ensuremath{\mathbf{#1}}}

\begin{document}

\title{\LARGE \bf
  A Minimalistic Approach to Segregation\\
  in Robot Swarms}

\author{
    Peter~Mitrano$^{1}$,
    Jordan~Burklund$^{1}$,
    Michael~Giancola$^{1}$,
    Carlo~Pinciroli$^{1}$%
\thanks{$^{1}$ Robotics Engineering, Worcester Polytechnic Institute, MA, USA. Email: {\sf cpinciroli@wpi.edu}}%
}

\maketitle
\thispagestyle{empty}
\pagestyle{empty}

\begin{abstract}
  We present a decentralized algorithm to achieve segregation into an arbitrary
  number of groups with swarms of autonomous robots. The distinguishing feature
  of our approach is in the minimalistic assumptions on which it is
  based. Specifically, we assume that (i) Each robot is equipped with a ternary
  sensor capable of detecting the presence of a single nearby robot, and, if
  that robot is present, whether or not it belongs to the same group as the
  sensing robot; (ii) The robots move according to a differential drive model;
  and (iii) The structure of the control system is purely reactive, and it maps
  directly the sensor readings to the wheel speeds with a simple `if' statement.
  We present a thorough analysis of the parameter space
  that enables this behavior to emerge, along with conditions for guaranteed
  convergence and a study of non-ideal aspects in the robot design.
\end{abstract}

% \begin{IEEEkeywords}
%   TBD
% \end{IEEEkeywords}

\section{Introduction}

Group formation is one of the most fundamental mechanisms a robot swarm must
exhibit~\cite{Brambilla2013}. Group formation can occur in several forms to
satisfy different requirements. Segregation is a particular type of group
formation in which the focus is on creating local aggregates of robots that
share a common property. Segregation can be seen as a precursor to object
sorting, task allocation, or self-assembly. For example, swarms may need to
split into arbitrary groups to diffuse and search different areas, or segregate
by skill or capability in order to form useful heterogeneous teams.

Segregation is an example of the broader class of spatially organizing
behaviors, whose purpose is to impose a structure in the environment (e.g.,
object clustering~\cite{gauci_clustering_2014}, collective
construction~\cite{Bolger2010}) or in the distribution of the robots (e.g.,
aggregation~\cite{shlyakhov_survey_2017}, pattern
formation~\cite{Pinciroli:DARS2016}, self-assembly~\cite{gross2008self}).

A recent line of research in spatially organizing behaviors focuses on the
\emph{minimal} assumptions a swarm of robots must fulfill in order to perform the
task. Johnson and Brown~\cite{johnson_evolving_2016} and Brown \emph{et
  al.}~\cite{brown_discovery_2018} characterized the set of possible behaviors
that can be obtained using primordial control strategies based on a simple
`if/then/else' structure, binary sensors, and differential-drive robots. Gauci
\emph{et al.} provided the specific conditions for the emergence of
aggregation~\cite{gauci_evolving_2014} and object
clustering~\cite{gauci_clustering_2014}, while St.-Onge \emph{et
  al.}~\cite{StOnge:IROS2018} studied the emergence of circular
formations. While more efficient control strategies have been proposed to
achieve these behaviors, studying the minimal assumptions required for their emergence is
an important step towards principled `swarm engineering' practices. In addition,
these minimal behaviors might offer last-resort solutions in case of sensor
failures in remote environments such as in planetary exploration missions.

This paper furthers this line of inquiry by studying the minimal assumptions for
$N$-class segregation to emerge from local, decentralized interactions among
robots. The term `$N$-class' refers to the creation of $N$ spatially distinct
groups. We show that, for segregation to emerge, it is sufficient to equip an
`if/then/else', differential drive robot with a \emph{ternary} sensor. This
sensor detects the presence of a robot in range. When a robot is detected, the
sensor can distinguish whether it is a \emph{kin}, i.e., it belongs to the same
group as the sensing robot, or a \emph{non-kin}, i.e., it belongs to a different
group. When multiple robots are in range, the sensor returns information on the
closest one of them.

The main contributions of this paper are \emph{(i)} A study of the parameter
space that enables the emergence of $N$-class segregation; \emph{(ii)} A study
of why convergence is guaranteed for the best parameter choice found; and
\emph{(iii)} An analysis of the robustness of the algorithm to non-idealities in
the robot design.

\section{Related Work}
Segregation is a common behavior in nature, and it can be observed across
scales. For example, cell segregation is a basic building block of embryogeneis
in tissue generation processes~\cite{batlle_molecular_2012,Steinberg1963}; while
social insects, such as ants, organize their brood into ring-like
structures~\cite{Franks1992}.

In robotics, segregation is a problem that has not received considerable
attention. The main methods that have been proposed so far are based on some
variation of the artificial potential approach~\cite{Spears2004}, which assumes
that the robots can detect each other and estimate relative distance vectors.

Gro\ss~\emph{et al.}~\cite{gross_segregation_2009} proposed an
algorithm inspired by the Brazil Nut effect, in which the robots form regular
layers simulating gravity by sharing a common direction. This study was later
extended to work on e-pucks robots~\cite{Chen2012}. To simulate gravity, this
approach requires the robots to share a common target vector, which can be
obtained through centralized controllers or a distributed consensus algorithm.

Kumar \emph{et al.}~\cite{kumar_segregation_2010} introduced the concept of
``differential potential'', whereby two robots experience a different artificial
potential depending on their being part of the same class or not. The
convergence of this approach is guaranteed for two classes, but when more
classes are employed local minima prevent segregation from emerging.

Santos \emph{et al.}~\cite{santos_segregation_2014} took inspiration
from~\cite{kumar_segregation_2010} to devise an approach based on the
Differential Adhesion Hypothesis, which states that kin cells tend to adhere
stronger than non-kin cells. Hower, one limitation is the assumption
that the robots have global knowledge about the positions of other robots.

To the best of our knowledge, this paper is the first to propose a segregation
algorithm that is not based on global information nor on communication and
sensing of multiple neighbors.

\section{Methodology}

\subsection{Problem Formulation}

\newcommand{\vL}{\ensuremath{v_{\text{left}}}}
\newcommand{\vR}{\ensuremath{v_{\text{right}}}}
\newcommand{\vaL}{\ensuremath{V_{\text{left}}}}
\newcommand{\vaR}{\ensuremath{V_{\text{right}}}}
\newcommand{\VM}{\ensuremath{V_{\text{max}}}}
\myparagraph{Motion model}
We consider a set of robots executing the same controller in a two-dimensional,
obstacle-free environment. The robots are equipped with two wheels for which
$[\vL,\vR]$ denote their \emph{normalized} linear speeds. By `normalized' we
mean that the speed values are in the range $[-1, 1]$. Using normalized speeds
allows us to reason in a general way over the specific speeds attainable by any
robot. To transform from normalized speeds $[\vL,\vR]$ into actual speeds
$[\vaL,\vaR]$, we introduce a parameter $\VM$ that denotes the maximum linear
speed possible with a specific robot and define
\begin{align}
  \vaL &= \VM \vL\\
  \vaR &= \VM \vR.
\end{align}
The robots' motion is modeled by the well-known differential-drive
equations~\cite{Dudek2010}
\begin{equation}
  \label{eq:diffdrive}
  \begin{aligned}
    x(t)      &=  \frac{l}{2}\frac{\vaR+\vaL}{\vaR-\vaL}\sin\left(\frac{\vaR-\vaL}{l}t\right)\\
    y(t)      &= -\frac{l}{2}\frac{\vaR+\vaL}{\vaR-\vaL}\cos\left(\frac{\vaR-\vaL}{l}t\right)\\
    \theta(t) &=  \frac{\vaR-\vaL}{l}t
  \end{aligned}
\end{equation}
where $t$ is time, $[\,x\;y\;\theta\,]^T$ is the pose of the robot, and $l$ is
the distance between the wheels.

\newcommand{\vPN}[2]{\ensuremath{v_{\text{#1}}}^{S=#2}}
\newcommand{\robot}[2]{%
  \filldraw[draw=#2,fill=#2!20] (#1) circle(5mm);
  \draw[draw=#2,->,-Stealth,rotate around={0:(#1)}] (#1) -- +(5mm,0);
  \fill[fill=gray!20] ($(#1)+(5mm,0)$) -- +( 45:1cm) -- +(-45:1cm) -- cycle;%
  \fill[fill=#2] ($(#1)+(5mm,0)$) circle (1mm);%
}
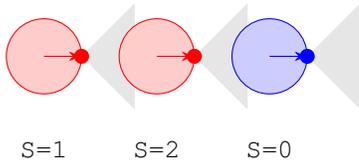
\begin{figure}[t]
  \centering
  \begin{tikzpicture}
    \coordinate (s0) at (0  cm,0cm);
    \coordinate (s1) at (1.5cm,0cm);
    \coordinate (s2) at (3  cm,0cm);
    \robot{s0}{red};
    \robot{s1}{red};
    \robot{s2}{blue};
    \node[below=of s0]{\texttt{S=1}};
    \node[below=of s1]{\texttt{S=2}};
    \node[below=of s2]{\texttt{S=0}};
  \end{tikzpicture}
  \caption{A diagram of the ternary sensor in which classes are depicted
    as colors. The left red robot detects a kin robot, so its sensor
    returns 1. The middle red robot detects a blue robot, so its sensor
    returns 2. The right robot detects no robot, so its sensor returns
    0.}
  \label{fig:sensor}
\end{figure}
\begin{algorithm}[t]
  \begin{algorithmic}
    \If {$S=0$} \State set wheel speeds to $\vPN{left}{0}$, $\vPN{right}{0}$
    \ElsIf {$S=1$} \State set wheel speeds to $\vPN{left}{1}$, $\vPN{right}{1}$
    \Else \State set wheel speeds to $\vPN{left}{2}$, $\vPN{right}{2}$
    \EndIf
  \end{algorithmic}
  \caption{The segregation control algorithm.}
  \label{alg:controller}
\end{algorithm}

\myparagraph{Sensor model}
The robots are also equipped with a ternary sensor that is able to detect the
presence of nearby robots and their ``kinness''. Two robots are \emph{kin} if
they belong to the same class (denoted by color in our experiments); they are
\emph{non-kin} otherwise. The sensor is assumed to have infinite range (we
consider non-infinite range in Sec.~\ref{section:beam_range}). As depicted in
Fig.~\ref{fig:sensor}, the sensor returns a reading $S=0$ when no robot is
detected, $S=1$ when a kin robot is detected, and $S=2$ when a non-kin robot is
detected. We allow for any number of classes, but the sensor need not
distinguish between different non-kin classes---it only detects whether a nearby
robot belongs to the same class or not.

\myparagraph{Control logic}
The control logic followed by the robots is formalized in
Alg.~\ref{alg:controller}. It is a simple `if/then/else' structure, which maps
the sensor readings $S$ directly into normalized wheel speeds $[\,\vL\;\vR\,]$. The
latter are the parameters whose value we intend to study, and they can be
encoded as a six-dimensional vector
$$
[\,
\vPN{left}{0}\;
\vPN{right}{0}\;
\vPN{left}{1}\;
\vPN{right}{1}\;
\vPN{left}{2}\;
\vPN{right}{2}\;
\,].
$$

\myparagraph{Objective}
The objective of our study is to find the values of the speed parameters for
which the robots group into clusters, such that all the robots of the same class
are packed into one cluster with no non-kin robots.

\subsection{Simulation Environment and Robots}

\myparagraph{Simulation platform} We utilized the ARGoS multi-robot
simulator~\cite{pinciroli_argos:_2012} to search for controller parameters and
evaluate them. ARGoS offers accurate models for several differential-drive
robots, such as the foot-bot~\cite{Bonani2010}, the Khepera
IV\footnote{https://www.k-team.com/khepera-iv}, and the
Kilobot~\cite{Rubenstein2012}.

\myparagraph{Robot platform}
For the simulated experiments we opted to use the foot-bot, because of the
possibility to utilize its range-and-bearing communication system as a base for
the ternary sensor. The advantage of using the range-and-bearing system is that
its model is simple and, as a consequence, a large number of simulations could
be completed in a short time. The range-and-bearing system allows two robots to
exchange messages when they are in direct line-of-sight; upon receiving a
message, a robot can also estimate the relative position of the sender. This
sensor, in principle, receives messages from all the nearby robots. To simulate
the ternary sensor, our robot controller kept the message of the closest
robot. The message payload was an integer that encoded the id of the group to
which the sender belonged. The range-and-bearing sensor is simulated through ray
casting. This allowed us to assume that the sensor reading is infinitely thin, a
choice that simplifies the mathematical analysis presented in
Sec.~\ref{sec:analysis}. However, in practical applications, the sensor can be
expected to cast a cone-shaped sensory range with non-zero aperture angle. We
explore the performance effect of various angles in
Sec.~\ref{sec:aperture_angle}.

\subsection{Grid Search}
\label{sec:gridsearch}

\myparagraph{Trial setup}
In order to exhaustively search the space of possible controllers, we conducted
a grid search of the 6-dimensional parameter space. Due to limited computational
resources, we were only able to search with a resolution of $7$ values per
parameter, which means in total we evaluated $7^6=117,649$ parameter sets. For
each parameter set, we tested 38 different initial configurations, with 100
simulated seconds for each trial. These initial configurations consisted of
uniformly random placement, clusters, and lines of robots distributed throughout
the environment. We chose to include some structured configurations (clusters
and lines) because we discovered that they affected significantly the
performance with respect to uniform random configurations. Hence, by explicitly
evaluating diverse initial configurations, we could better estimate the best
parameter values in the general case. Examples of these starting configurations
can be seen in our supplementary videos:
\href{https://www.youtube.com/playlist?list=PL9HqYJ1IkIKVX9EsT5BY9LnBsBPTjc5bB}{https://goo.gl/z8UAuB}.

\myparagraph{Clusters}
To define our cost function, we first establish the notion of `cluster',
intuitively defined as an island of connected kin robots.  Denoting with $r$ the
radius of the body of a robot, and defining
$\vec{p}_i(t) = [\,x_i(t)\;y_i(t)\,]$ as the $x$ and $y$ coordinates of the
robot at time $t$, we consider any two kin robots to be connected if
$$
\lVert\vec{p}_i(t) - \vec{p}_j(t)\rVert \le 2r + \epsilon \qquad (i \ne j, \epsilon \in \mathbb{R}^+).
$$
In our experiments, we set $\epsilon = \SI{5}{\centi\meter}$.  We find clusters
by first constructing an adjacency matrix and then performing a depth-first
search.  Since we are interested in segregating the robots in $N$ classes, the
final result of a trial is expected to be a set of $N$ distinct clusters
composed of kin robots.

\myparagraph{Cost function}
To measure the difference between the ideal, perfectly segregated result and any
configuration achieved by the robots over time, we first calculate, for every
class $i$, the number of robots in the largest cluster $c_i(t)$ formed by robots of class $i$.  Since,
in principle, different classes might involve different numbers of robots, in
our cost function we employ the ratio
$$
\gamma_i(t) = -\frac{c_i(t)}{C_i}
$$
This ratio should be maximized, so the negative sign assigns larger clusters a lower, more negative, cost.
Here $C_i$ is the number of robots that belongs to class $i$. At each time step, the cost is
$$
\gamma(t) = \frac{1}{N}\sum_{i=1}^N\gamma_i(t)
$$
Our complete cost function is then
\begin{equation}
  \label{eq:cost_function}
  c_{\text{total}} =  \sum_{t=0}^{T-1} t\gamma(t)
\end{equation}
in which we denote the total trial time (100 seconds) with $T$. The effect of multiplying $\gamma_i(t)$ by $t$ is to highlight the
emergence of clusters as the trial time proceeds: we cannot expect large
clusters to be present at the beginning of a trial, but good parameter settings
should grow (and maintain) clusters over time. In our experiments, we found that
$c_{\text{total}}$ correctly assigns cost in most scenarios, thus fitting
well our analysis purposes. However, this cost function considers a straight
line of robots to be a cluster, and as such it might not be ideal for scenarios
in which the clusters are required to be tight.

\begin{figure}[t]
  \centering
  \includegraphics[width=0.32\linewidth]{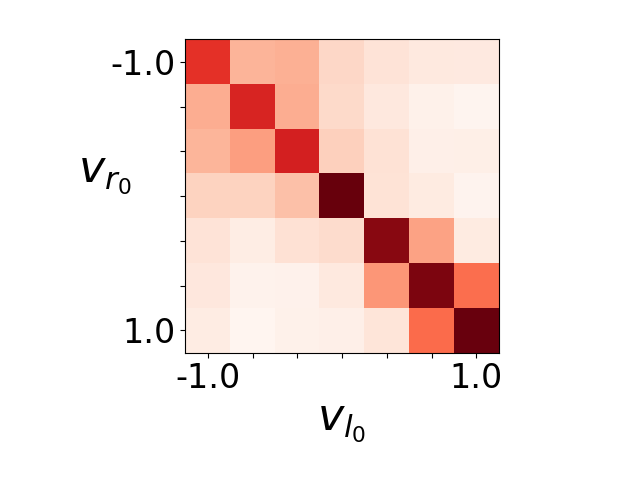}
  \includegraphics[width=0.32\linewidth]{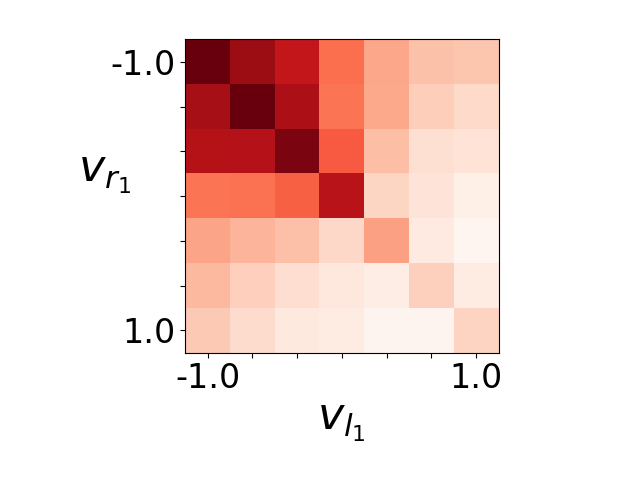}
  \includegraphics[width=0.32\linewidth]{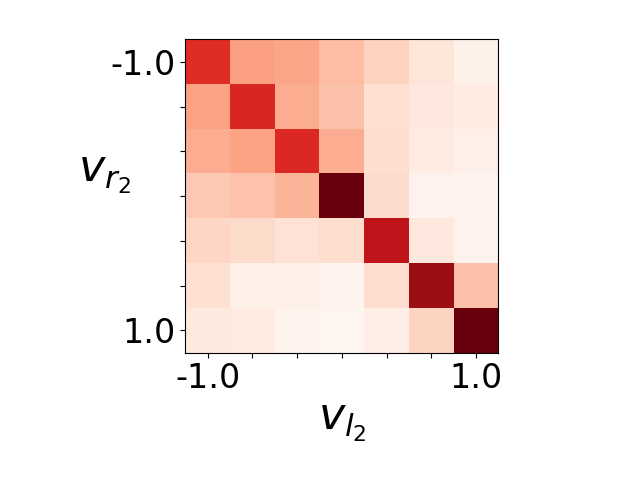}
  \includegraphics[width=0.32\linewidth]{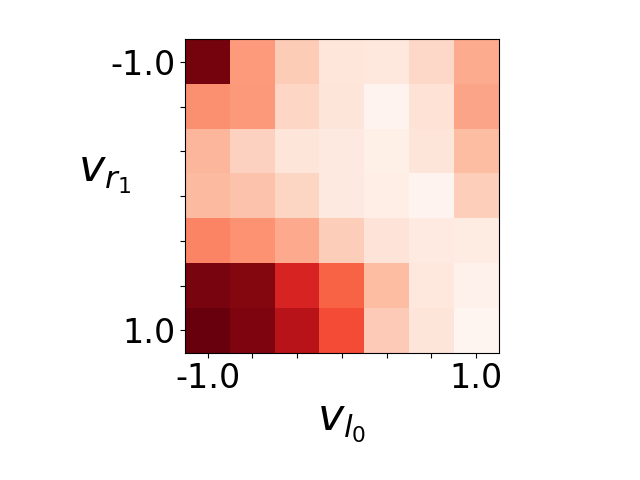}
  \includegraphics[width=0.32\linewidth]{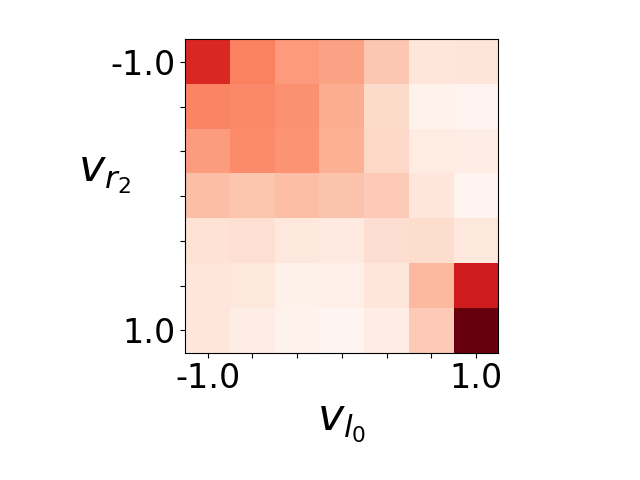}
  \includegraphics[width=0.32\linewidth]{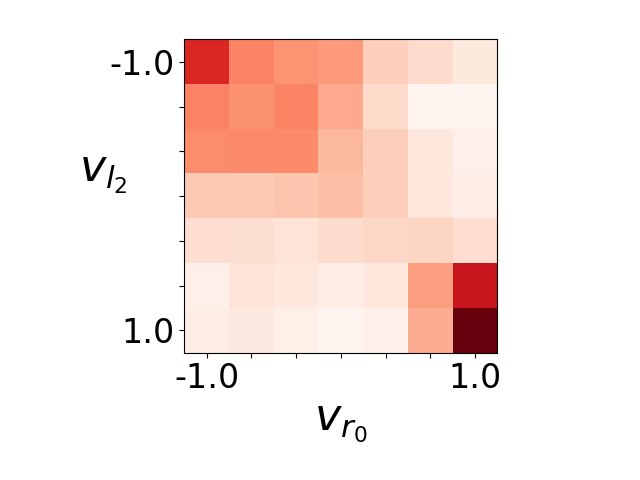}
  \includegraphics[width=0.32\linewidth]{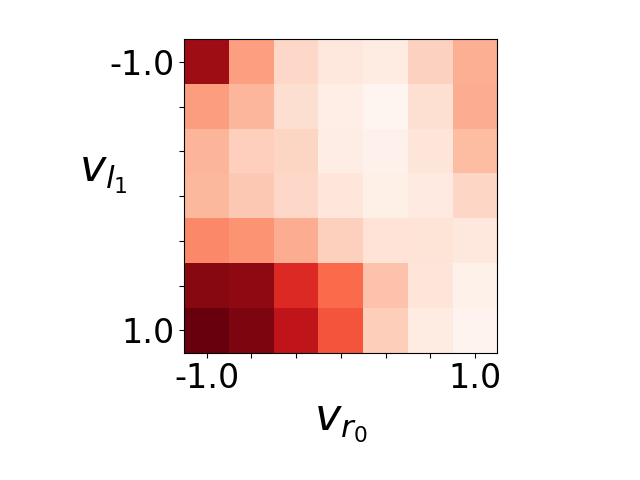}
  \includegraphics[width=0.32\linewidth]{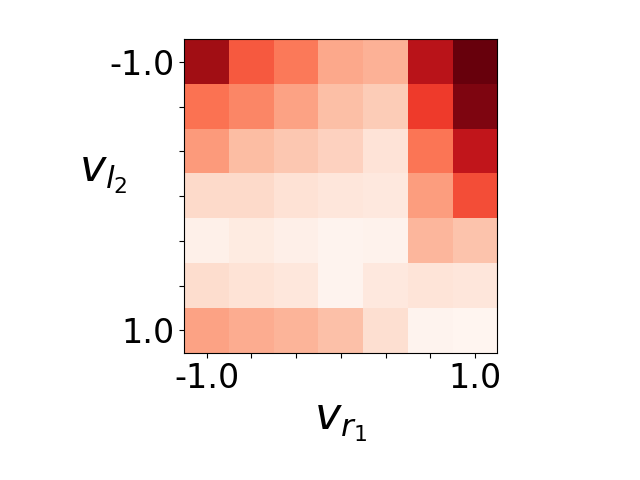}
  \includegraphics[width=0.32\linewidth]{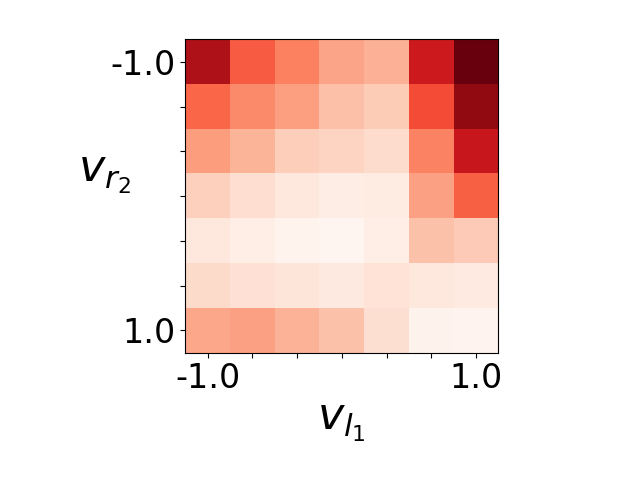}
  \includegraphics[width=0.32\linewidth]{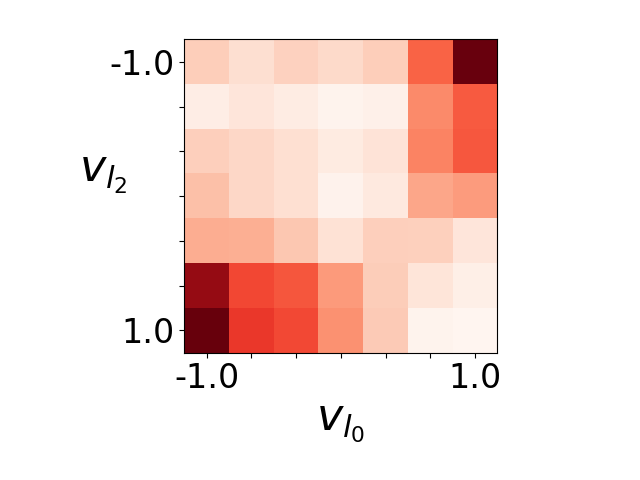}
  \includegraphics[width=0.32\linewidth]{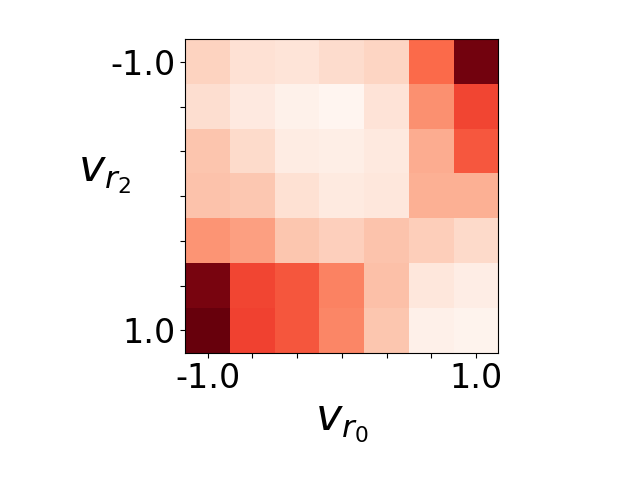}
  \includegraphics[width=0.32\linewidth]{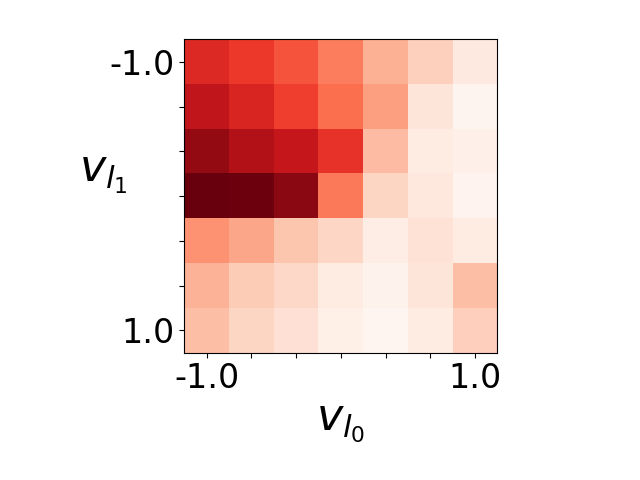}
  \includegraphics[width=0.32\linewidth]{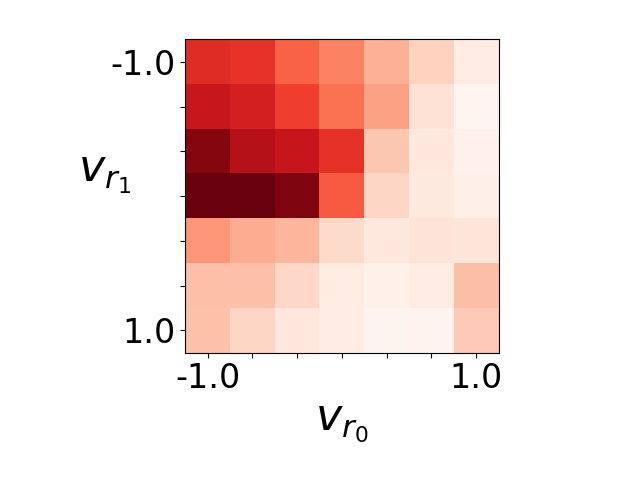}
  \includegraphics[width=0.32\linewidth]{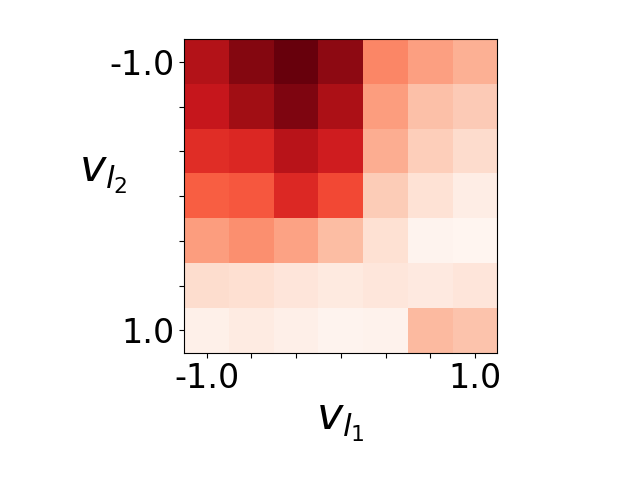}
  \includegraphics[width=0.32\linewidth]{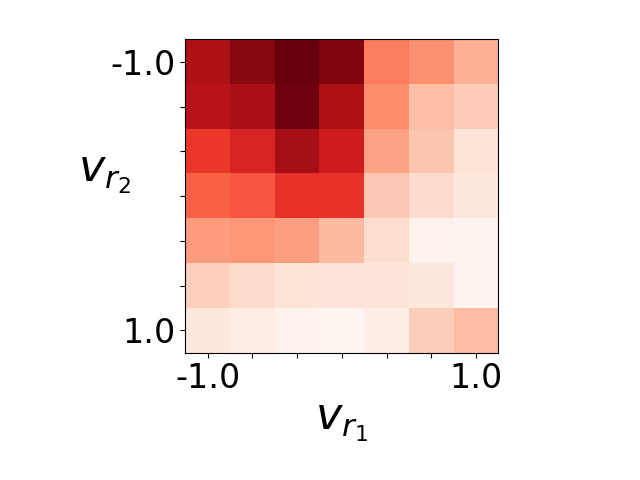}
  \caption{Heatmaps that relate relevant pairs of wheel speeds.}
  \label{fig:gridsearch}
\end{figure}

\section{The Emergent Behavior}
\myparagraph{Visualizing grid search}
The results of grid search are reported in Fig.~\ref{fig:gridsearch}. Because
the search space is six-dimensional, we chose to visualize it by plotting every
pair of parameters against each other. For example, we consider how the cost
changes as $\vPN{left}{0}$ and $\vPN{right}{0}$ change. As an example of reading
these plots, we can tell from the plot of $\vPN{left}{1}$ and $\vPN{right}{1}$
that there were no good controllers where the left and right wheel speeds were
equal and negative (dark squares in the upper left), and that the best
controllers had slightly unequal values close to 1 (lightest squares in the
bottom left). These plots also convey the presence of sharp discontinuities
where performance changes dramatically.

\myparagraph{The emergent behavior}
After running the grid search, the parameters with the lowest mean cost across all 38
configurations was
\begin{equation}
\label{eq:controller}
[1, -\sfrac{2}{3}, \sfrac{1}{3}, 1, 1, 0].
\tag{C}
\end{equation}

\begin{figure}[t]
  \centering
  \includegraphics[width=0.5\linewidth]{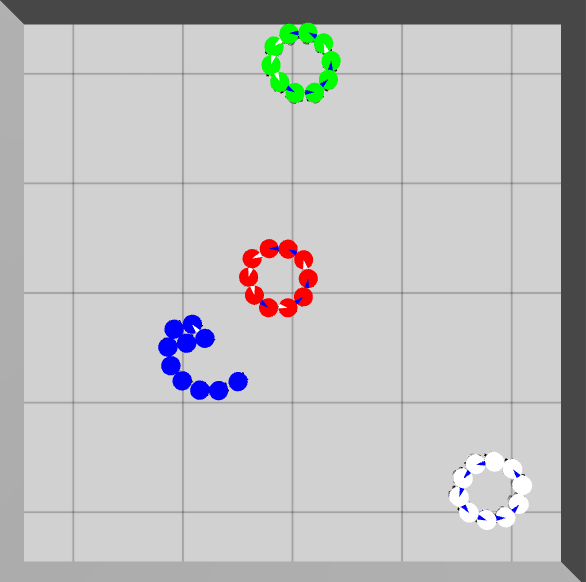}
  \caption{The segregation behavior found by grid search consists in the
    formation of homogeneous rings or spiraling structures that keep spinning
    over time. The rings tend to grow over time, disband, and reform.}
  \label{fig:rings}
\end{figure}
The resulting behavior is for a robot to turn away from kin, but turn the
opposite way when the robot sees nothing or non-kin. This behavior amounts to
robots zig-zagging in a line towards their kin. As discussed
in~\cite{StOnge:IROS2018}, when multiple robots execute this zig-zag behavior,
spinning rings are formed. The remarkable aspect, in our case, is that spinning
rings of \emph{kin} robots emerge, eventually segregating the swarm in
homogeneous groups. An example of this phenomenon as observed in simulated
experiments is reported in Fig.~\ref{fig:rings}. We also noted that occasionally the
robots form spinning spiral shapes, which can also be seen in
Fig.~\ref{fig:rings}.

\myparagraph{Interesting phenomena}
It is important to note that, although we originally hoped that segregated
clusters would be tightly packed, none of the top scoring parameters according
to our cost function form tightly packed clusters. There \emph{are} parameters
which achieve tightly packed segregation, but they do so extremely slowly, and
therefore our cost function correctly penalized them heavily for being too slow
to be useful. In addition, we observed that the spinning rings formed by the
robots tended to expand over time, disband into smaller structures, and
eventually reappear. In addition, if a ring is disturbed by non-kin robots
passing through it, the ring is disrupted but it eventually reforms. The growing
and self-repairing dynamics is compatible with the findings in the work of
St.-Onge \emph{et al.}~\cite{StOnge:IROS2018}, in which ring formation is
decomposed in four simpler behavioral traits: scouting, chaining, looping, and
merging. This segregation behavior follows the same phases, and it constitutes a
further example of structured space-time coordination arising from minimalistic
assumptions on the capabilities of the robot. To better appreciate the dynamics
of this behavior, we invite the interested reader to watch the videos at
\href{https://www.youtube.com/playlist?list=PL9HqYJ1IkIKVX9EsT5BY9LnBsBPTjc5bB}{https://goo.gl/z8UAuB}.

\section{Behavior Analysis}
\label{sec:analysis}
Using the parameter settings~\eqref{eq:controller} as a basis, we now analyze
the emerging behavior with the purpose of explaining how and why it emerges. In
particular, we discuss the conditions under which segregation is guaranteed.
\newcommand{\ICC}{\ensuremath{\text{ICC}}}

In the proofs, we employ the well-known equations that govern the instantaneous
radius of curvature $R$ and rotation speed $\omega$ of the path followed a
differential-drive robot with $l$ denoting the interwheel distance~\cite{Dudek2010}:
\begin{equation}
  \label{eq:Ricc}
  \begin{aligned}
    R &= \frac{l}{2}\frac{\vaR + \vaL}{\vaR - \vaL} = \frac{l}{2}\frac{\vR + \vL}{\vR - \vL}\\
    \omega &= \frac{\vaR - \vaL}{l} = \frac{\VM(\vR - \vL)}{l}.
  \end{aligned}
\end{equation}

\begin{theorem}[Scouting]
  When a robot $i$ does not detect a kin, it turns clockwise until it finds one.
\end{theorem}
\begin{proof}
  The proof derives from the observation of the speeds in
  \eqref{eq:controller}. When $S=0$ (no robot detected) and $S=2$ (non-kin
  detected), the left wheel speed is larger than the right one, producing a
  circular clockwise motion. Conversely, when $S=1$ (a kin is detected), the
  right wheel speed is larger than the left one, and the robot turns
  counterclockwise.
\end{proof}

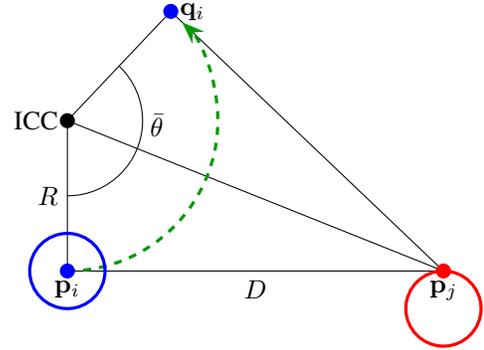
\begin{figure}[t]
  \centering
  \begin{tikzpicture}
    % Coordinates
    \coordinate(Pi)  at (0,0);
    \coordinate(Pj)  at ($(Pi)+(5cm,0cm)$);
    \coordinate(ICC) at ($(Pi)+(0cm,2cm)$);
    \coordinate(Qi)  at ($(Pj)!5cm!-21.812:(ICC)$);
    % Label R
    \draw (ICC) -- (Pi) node[left,midway]{$R$};
    % Label delta
    \draw (Pi) -- (Pj) node[below,midway]{$D$};
    % Label theta
    \draw ($(ICC)-(0,1cm)$) arc (-90:46.376:1cm) node[pos=0.75,below right]{$\bar{\theta}$};
    % Qi to ICC
    \draw (Qi) -- (ICC);
    % Pj to ICC
    \draw (Pj) -- (ICC);
    % Pj to Qi
    \draw (Pj) -- (Qi);
    % Trajectory and ICC
    \draw[dashed,very thick,green!60!black,-{Stealth[]},shorten >=2mm] ($(ICC)-(0,2cm)$) arc (-90:46.376:2cm);
    \fill (ICC) circle (1mm);
    \node[left] at (ICC) {$\ICC$};
    % Blue robot
    \draw[very thick,blue] (Pi) circle (5mm);
    \fill[blue] (Pi) circle (1mm);
    \node[below] at (Pi) {$\vec{p}_i$};
    % Qi
    \fill[blue] (Qi) circle (1mm);
    \node[right] at (Qi) {$\vec{q}_i$};
    % Red robot
    \draw[very thick,red] ($(Pj)-(0,5mm)$) circle (5mm);
    \fill[red] (Pj) circle (1mm);
    \node[below] at (Pj) {$\vec{p}_j$};
  \end{tikzpicture}
  \caption{A diagram for the geometry of motion towards kin robots. The red
    circle ($\vec{p}_j$) depicts the detected kin neighbor; the blue circle
    indicates the robot whose motion is being modeled ($\vec{p}_i$); the dashed
    green arc is the path followed by the robot, up to the limit point
    $\vec{q}_i$. Beyond the latter point, the blue robot ceases to move closer
    to the red robot.}
  \label{fig:movetokin}
\end{figure}
\begin{lemma}[Motion towards kin when kin detected]
  A robot moves towards a kin robot if
  \begin{equation}
    \label{eq:movetokin}
    \VM\Delta t \le 3l\tan^{-1}\left[\frac{\sqrt{3}r}{l}\right].
    \tag{S1}
  \end{equation}
\end{lemma}
\begin{proof}
  When a robot encounters a kin, it picks the speeds
  $\vPN{left}{1} = \sfrac{1}{3}$ and $\vPN{right}{1} = 1$. This corresponds to a
  path that arcs counterclockwise, as reported in the diagram of
  Fig.~\ref{fig:movetokin}. We indicate the position of the moving robot as
  $\vec{p}_i$ and the position of the kin as $\vec{p}_j$. For robot $i$ to move
  towards $j$, its trajectory path must be an arc that does not move past a
  limit point $\vec{q}_i$ because, beyond this point, the distance between $i$
  and $j$ would increase. We can express this condition as
  $$|\omega|\Delta t\le \bar{\theta}$$
  where $\bar{\theta}$ is the angle of the arc connecting $\vec{p}_i$ and $\vec{q}_i$.

  Reasoning on the triangle formed by $\vec{p}_i$, $\vec{p}_j$ and the $\ICC$,
  we can calculate
  \begin{equation}
    \label{eq:theta2}
    \bar{\theta} = 2\tan^{-1}\frac{D}{|R|}
  \end{equation}
  where we defined $D = \lVert\vec{p}_i - \vec{p}_j\rVert$. The right hand side of this
  inequality monotonically increases with $D$, hence to find the most strict
  condition we need to consider the smallest value of $D$ possible. This value
  corresponds to the situation in which $i$ and $j$ are tangent to each other;
  in this case, $D=\sqrt{3}r$, where $r$ indicates the robot radius (assumed
  identical for both robots). Using \eqref{eq:Ricc} with the values of
  $\vPN{left}{1}$ and $\vPN{right}{1}$, we obtain:
  \begin{align*}
    R &= \frac{l}{2}\frac{1+\sfrac{1}{3}}{1-\sfrac{1}{3}} = l\\
    \omega &= \frac{\VM(1-\sfrac{1}{3})}{l} = \frac{2}{3}\frac{\VM}{l}.
  \end{align*}
  Plugging these expressions in \eqref{eq:theta2}, we obtain the statement.
\end{proof}

\begin{lemma}[Motion towards kin when nothing detected]
  Assume that a robot $i$ saw a kin $j$ at time $t-1$, has performed one step
  with $\vPN{left}{1}$ and $\vPN{right}{1}$, and at time $t$ it detects no robot.
  Robot $i$ moves towards the kin robot if
  \begin{equation}
    \label{eq:movetonone}
    \VM\Delta t \le \frac{6}{5}l\tan^{-1}\left[\frac{10\sqrt{3}r}{l}\right].
    \tag{S0}
  \end{equation}
\end{lemma}
\begin{proof}
  This proof follows the same reasoning as in Lemma 1 with the wheel speeds
  $\vPN{left}{0} = 1$ and $\vPN{right}{0} = -\sfrac{2}{3}$.
\end{proof}

\begin{lemma}[Motion towards kin when non-kin detected]
  Assume that a robot $i$ saw a kin $j$ at time $t-1$, has performed one step
  with $\vPN{left}{1}$ and $\vPN{right}{1}$, and at time $t$ it detects a non-kin.
  Robot $i$ moves towards the kin robot if
  \begin{equation}
    \label{eq:movetononkin}
    \VM\Delta t \le 2l\tan^{-1}\left[\frac{2\sqrt{3}r}{l}\right].
    \tag{S2}
  \end{equation}
\end{lemma}
\begin{proof}
  This proof follows the same reasoning as in Lemma 1 with the wheel speeds
  $\vPN{left}{2} = 1$ and $\vPN{right}{2} = 0$.
\end{proof}

\begin{figure}[t]
  \centering
  \includegraphics[width=0.8\linewidth]{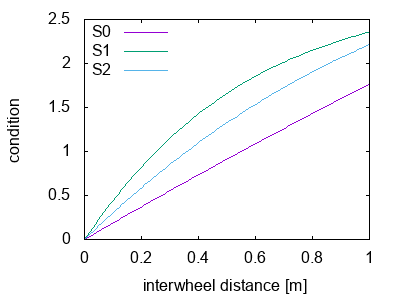}
  \caption{A graphical proof that \eqref{eq:movetonone} is the strictest
    condition. The graph relates the length of the interwheel distance to the
    value of the conditions, for a constant value of the robot body radius.}.
  \label{fig:conditions}
\end{figure}
\begin{theorem}[Chaining]
  A robot $i$ eventually follows a kin $j$ if \eqref{eq:movetonone} holds.
\end{theorem}
\begin{proof}
  Because of Theorem 1, a robot that has not detected a kin rotates clockwise
  until it finds one. After this, the robot turns counterclockwise and, if
  \eqref{eq:movetokin} holds, it steps closer to the kin. When the robot cannot
  detect the kin anymore it turns clockwise. If both
  \eqref{eq:movetonone} and \eqref{eq:movetononkin} hold, then it is guaranteed
  robot $i$ moves closer to the kin. Eventually robot $i$ will see the kin
  again, and the cycle continues. This reasoning can be repeated for any pair of
  robots, and a chain self-sustains if the three conditions are satisfied at the
  same time. This occurs when the strictest among them holds. As shown
  graphically in Fig.~\ref{fig:conditions}, the strictest condition is
  \eqref{eq:movetonone}.
\end{proof}

\begin{theorem}[Looping]
  A chain of kin robots eventually forms a loop.
\end{theorem}
\begin{proof}
  When a chain is formed, the robot in front either detects a kin, in which case
  it follows it (thus growing the chain); or it does not detect a kin, in which
  case it moves clockwise until it detects one of the kin robots that follow it
  in the chain. Since every robot behind the front robot performs the chaining
  behavior, this results in the chain looping on itself until the tail is
  reached and the loop is closed.
\end{proof}

\section{Experimental Results}

\subsection{Scalability Study} \label{section:scalability}

In this experiment we investigate how the segregation behavior scales with the
number of classes and the number of robots in the environment. We varied the
number of classes from 1 to 25 and ran 100 trials with robots uniformly randomly
distributed.

\myparagraph{Fixed number of robots per class}
We studied the case in which every class has 10 robots. The results of this are
plotted in Fig.~\ref{fig:num_classes_10}. We observe that the cost increases
with the number of classes. This occurs because, as the number of classes
increases, the density of the robots increases too. Hence, line-of-sight
occlusions between robots are more likely, navigation is more difficult, and the
clusters do not have a chance to coalesce.

\myparagraph{Fixed total number of robots}
We considered the scenario in which a fixed number of robots is split into an
increasing number of classes. We set the number of robots to 100, so with 25
classes 4 robots were still assigned to each class. As reported in
Fig.~\ref{fig:num_classes_100}, the cost is high for small numbers of classes
but its value decreases fast and eventually oscillates lightly. The initial high
cost is due to the fact that, with 100 robots to divided among few classes, it
is difficult for every robot to join the same cluster. The clusters, instead,
tend to form large islands of kins. As the number of classes increases, we
conjecture that the oscillations are an artifact of the random initial
conditions of each experiment.

\begin{figure}[t]
  \centering
  \includegraphics[width=1\linewidth]{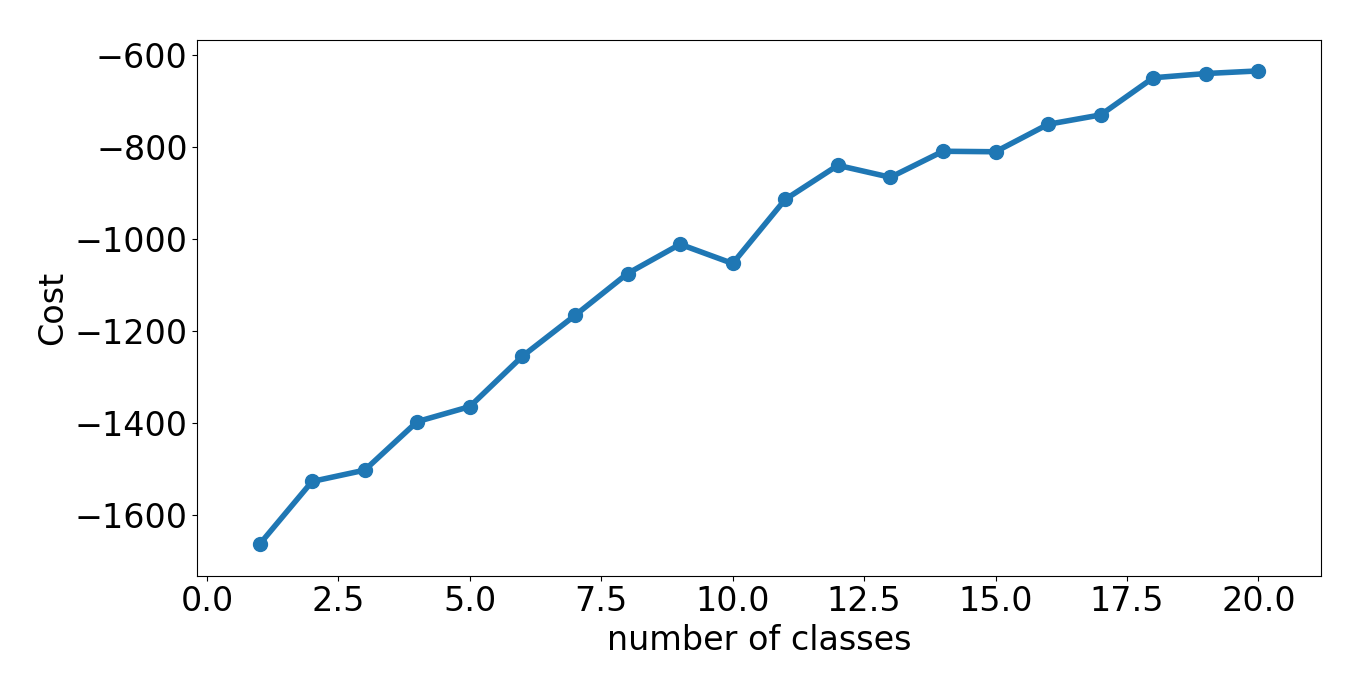}
  \caption{The average cost over 100 trials with $N$ classes, 10 robots per class.}
  \label{fig:num_classes_10}
\end{figure}

\begin{figure}[t]
  \centering
  \includegraphics[width=1\linewidth]{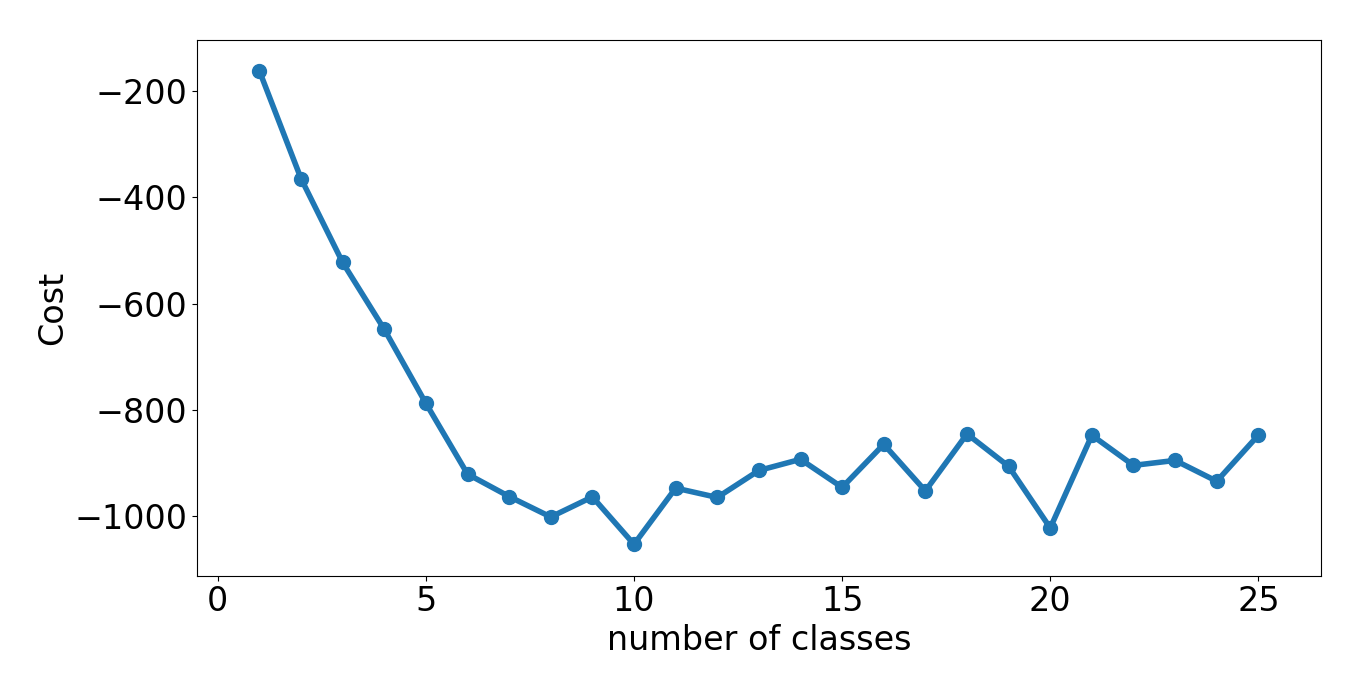}
  \caption{The average cost over 100 trials with 100 robots divided into $N$ classes.}
  \label{fig:num_classes_100}
\end{figure}

\subsection{The Effect of Implementation Details of the Sensor} \label{section:sensor_impl}

We observed that the implementation details of the sensor have a significant
effect on the behavior of the controller.

Initially, our method for determining sensor state from the simulated
range-and-bearing system was to consider all the robots within some small angle
in front of the robot and pick the closest one. This is very similar to what
would be provided by a real-world camera that uses colored skirts on each robot
and picks the largest blob as the robot to be detected. This sensor
implementation works well and was used in the grid search experiments. However,
we found later that, if the robots instead always prefer to react to kin over
non-kin, larger rings form more quickly and robustly. For example, if there are
two robots within the field of view of a robot's sensor and the non-kin robot is
closer, the robot would ignore it and execute the $S=1$ logic, which drives the
robot towards the farther kin robot. Exploring exactly which of the various
implementation details have what effect on cost is left for future work.

\subsection{The Effect of the Beam Angle} \label{sec:aperture_angle}

On a real robot, there must be some finite beam angle to the theoretically
line-of-sight sensor. We ran 100 trials in simulation with uniformly random
initial distributions of 40 robots with various beam
angles. Fig.~\ref{fig:beam_angle} shows the results, along with a diagram
showing how we define beam angle. The best beam angle we tested was \ang{15},
and angles smaller or larger became progressively worse. We found that at lower
beam angles, it was possible for a robot to become stuck in groups of two or
three, and the robots spent all their time looking at each other instead of
peeking around them to find kin. At larger angles, we suspect the behavior fails
because larger beam angles cause the rings to enlarge faster, which in turn
causes the rings to be so large that they are not considered a cluster anymore
by our cost function.

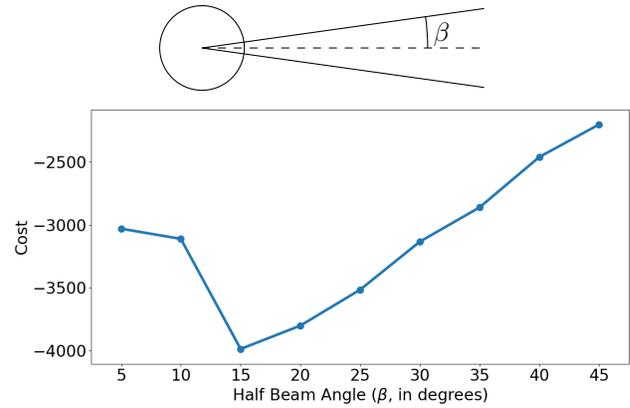
\begin{figure}[t]
  \centering
  \begin{tikzpicture}[scale=0.75]
    \draw (0,0) circle (0.75);
    \draw (0,0) -- (5,0.7);
    \draw[dashed] (0,0) -- (5,0);
    \draw (0,0) -- (5,-0.7);
    \draw (4,0) arc [radius=4, start angle=0, end angle=8];
    \node at (4.25,0.25) {$\beta$};
  \end{tikzpicture}
  \includegraphics[width=1\linewidth]{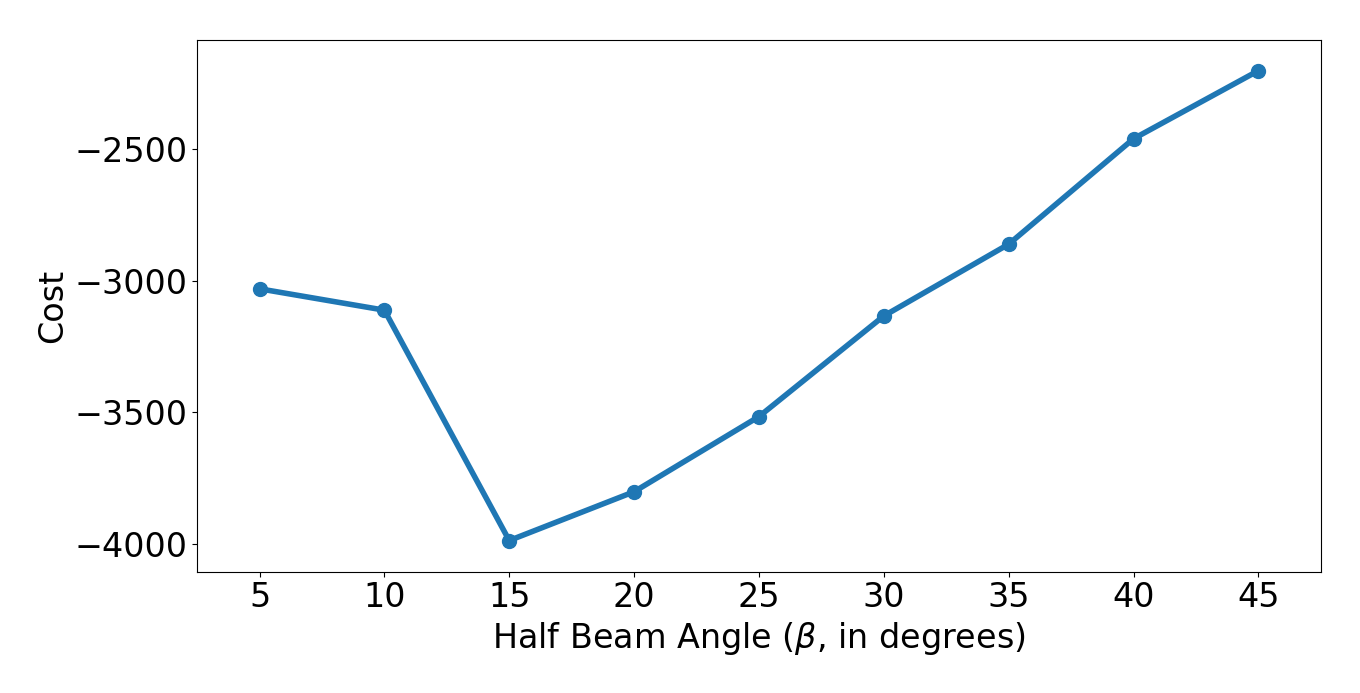}
  \caption{A \ang{15} degree half beam angle is best for segregation. Lower cost is better.}
  \label{fig:beam_angle}
\end{figure}

\subsection{The Effect of Beam Length} \label{section:beam_range}

We consider what happens if the theoretically infinite-range sensor has finite
range. We use \ang{15} half beam angle and the same experimental setup as with
the beam angle experiments. We consider the maximum range of the sensor as the
diagonal length of the square in which the robot are initially distributed. In
all our experiments, this square was \SI{5}{\meter} on each side, so we consider
a range of \SI{7.07}{\meter} to be effectively unlimited. We report the costs
for beam ranges as a fraction of this maximum range. As shown in
Fig.~\ref{fig:beam_range}, a beam range of 35\% of the theoretical maximum
performs just as well as an infinite sensor. Below this, the performance
degrades. However, even a beam range of 7\% of the maximum is more effective
than zero range at segregation.

\begin{figure}[t]
  \centering
  \includegraphics[width=1\linewidth]{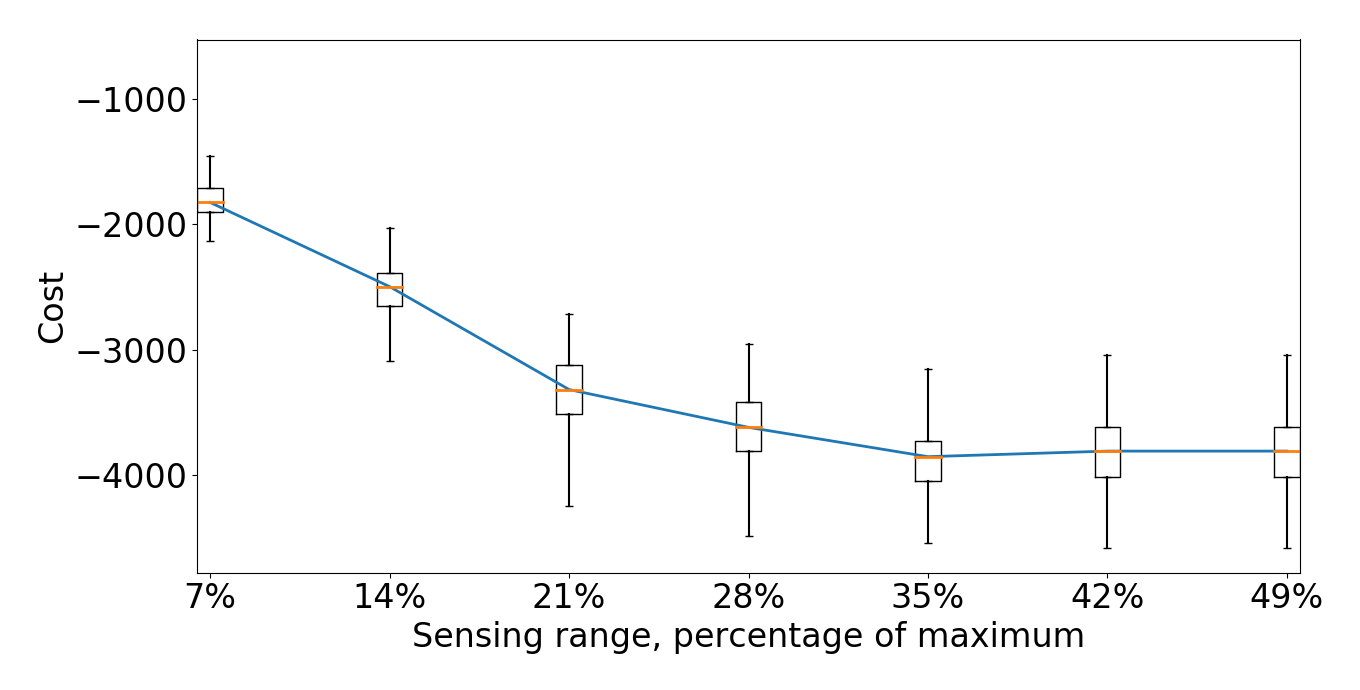}
  \caption{Segregation is robust to small sensor beam ranges. The performance at 35\% of maximum range is indistinguishable from infinite range.}
  \label{fig:beam_range}
\end{figure}

\section{Conclusion}

In this paper, we show how robots with only a ternary sensor and a controller
which maps sensor readings to wheel speeds is capable of $N$-class segregation.
This controller is invariant to the number of classes, and using the best found
parameters we are able to construct formal guarantees on the emergent behavior.
We performed a grid search to learn about the full parameter space, and we
investigated the effect of sensor implementation details and the number of
robots and classes on performance. Our findings indicate that robust segregation
with non-ideal sensors in reality is possible, although not guaranteed.

\bibliographystyle{IEEEtran}
\bibliography{RBE595.bib}

\end{document}